\newif\ifarXiv
\arXivfalse
\newif\ifWP
\WPfalse

\newif\ifLATIN
\LATINfalse

\arXivtrue			

\ifarXiv\LATINtrue\fi	

\newif\ifnotLATIN	
\notLATINtrue
\ifLATIN\notLATINfalse\fi

\ifarXiv
  \newcommand{\DFI}{GTP7full}
  \newcommand{\DFII}{vovk/etal:arXiv0505083}
  \newcommand{\DFVIII}{vovk:arXiv0606093}
  \newcommand{\AAVI}{vovk:arXiv0607067}
\fi
\ifWP
  \newcommand{\DFI}{GTP7full}
  \newcommand{\DFII}{GTP8}
  \newcommand{\DFVIII}{GTP17}
  \newcommand{\AAVI}{vovk:arXiv0607067}
\fi

\ifnotLATIN
  \newcommand{\Levin}{levin:1976uniform}
\fi
\ifLATIN
  \newcommand{\Levin}{levin:1976uniform-latin}
\fi

\ifarXiv
\documentclass{article}
\usepackage{amsmath,amsfonts,latexsym,epsfig}
\fi

\ifWP
\documentclass[toc]{gtarticle}
\usepackage{amsmath,amsfonts,latexsym,epsfig}

\fi

\allowdisplaybreaks[4]

\ifnotLATIN
\input{OT2enc.def}

\fi

\newcommand{\Vladimir}{Vladimir }
\newcommand{\DOT}{.}

\DeclareMathOperator{\diam}{diam}
\DeclareMathOperator{\dist}{dist}
\DeclareMathOperator{\co}{co}
\DeclareMathOperator{\D}{D}
\newcommand{\dd}{\textrm{d}}

\newcommand{\st}{\mathop{|}}

\newcommand{\K}{\mathcal{K}}		
\newcommand{\F}{\mathcal{F}}		
\newcommand{\AAA}{\mathcal{A}}		
\newcommand{\BBB}{\mathcal{B}}		
\newcommand{\CCC}{\mathcal{C}}		
\newcommand{\PPP}{\mathcal{P}}		
\newcommand{\SSS}{\mathcal{S}}		

\newcommand{\PPPfin}{\mathcal{P}^{\text{fin}}}		

\newcommand{\bbbr}{\mathbb{R}}

\newtheorem{corollary}{Corollary}
\newtheorem{theorem}{Theorem}
\newenvironment{Theorem}[1]
  {\trivlist\item[\hskip\labelsep\textbf{#1}]\it}
  {\endtrivlist}
\newenvironment{proof}
  {\trivlist\item[\hskip\labelsep\textbf{Proof}]}
  {\endtrivlist}
\newenvironment{Proof}[1]
  {\trivlist\item[\hskip\labelsep\textbf{Proof #1:}]}
  {\endtrivlist}
\newcommand{\boxforqed}{\rule{.3em}{1.5ex}}
\newcommand{\qedtext}{\unskip\nobreak\hfil
  \penalty50\hskip1em\null\nobreak\hfil\boxforqed
  \parfillskip=0pt\finalhyphendemerits=0\endgraf}

\newenvironment{remark*}
  {\trivlist\item[\hskip\labelsep{\bfseries Remark}]\relax}
  {\endtrivlist}

\newlength{\IndentI}
\newlength{\IndentII}
\newlength{\IndentIII}
\newlength{\IndentIV}
\newlength{\WidthI}
\newlength{\WidthII}
\newlength{\WidthIII}
\newlength{\WidthIV}
\ifarXiv
\title{Continuous and randomized defensive forecasting:\\unified view}
\author{Vladimir Vovk\\
\texttt{vovk{\rm@}cs.rhul.ac.uk}\\
\texttt{http://vovk.net}}
\fi

\ifWP
\title{Continuous and randomized defensive forecasting: unified view}
\author{Vladimir Vovk}

\fi

\begin{document}
\maketitle
\begin{abstract}
  Defensive forecasting is a method of transforming laws of probability
  (stated in game-theoretic terms as strategies for Sceptic)
  into forecasting algorithms.
  There are two known varieties of defensive forecasting:
  ``continuous'', in which Sceptic's moves are assumed
  to depend on the forecasts in a (semi)continuous manner
  and which produces deterministic forecasts,
  and ``randomized'', in which the dependence of Sceptic's moves on the forecasts is arbitrary
  and Forecaster's moves are allowed to be randomized.
  This note shows that the randomized variety can be obtained from the continuous variety
  by smearing Sceptic's moves to make them continuous.
  \ifarXiv

  \smallskip

  \textbf{New as compared to version 1 (17 August 2007) of this report:}
    The assumption of version 1 that the outcome space $\Omega$ is finite is relaxed,
    and now it is only assumed to be compact.
    In the case where $\Omega$ is finite,
    it is shown that Forecaster can choose his randomized forecasts
    concentrated on a finite set of cardinality at most $\lvert\Omega\rvert$.
  \fi
\end{abstract}

\section{Introduction}

The continuous variety of defensive forecasting
was essentially introduced by Levin
\cite{\Levin},
but was later rediscovered by Kakade and Foster \cite{kakade/foster:2004}
and Takemura \emph{et al.}\ \cite{\DFII}.

The randomized variety was introduced
(in the case of von Mises's version of the game-theoretic approach to probability)
by Foster and Vohra \cite{foster/vohra:1998}
and further developed by, among others, Sandroni \emph{et al.}\ \cite{sandroni/etal:2003};
these papers, however, were only concerned with asymptotic calibration.
Non-asymptotic versions of the randomized variety
were proposed by Sandroni \cite{sandroni:2003}
(based on standard measure-theoretic probability)
and Vovk and Shafer \cite{\DFI}
(based on game-theoretic probability).
Kakade and Foster \cite{kakade/foster:2004} noticed
that some calibration results require very little randomization
(this will be an important aspect of our Theorem \ref{thm:randomized}).

This note states two simple results about defensive forecasting,
Theorem \ref{thm:continuous} about the continuous variety
and Theorem \ref{thm:randomized} about the randomized variety.
The proof of Theorem \ref{thm:randomized} is obtained
from the proof of Theorem \ref{thm:continuous}
by blurring Sceptic's moves.

In our informal discussions
we will be assuming
that the set $\Omega$ of all possible outcomes is finite,
although we will try to make mathematical statements
as general as possible.
The reader who is only interested in the main ideas
might choose to specialize
Theorems \ref{thm:continuous} and \ref{thm:randomized} and their proofs
to the case of finite $\Omega$.

\section{Continuous defensive forecasting}\label{sec:continuous}

Let $\Omega$ (the \emph{outcome space}) be a compact
(i.e., a compact Hausdorff topological space)
equipped with the Baire $\sigma$-algebra
and $\PPP(\Omega)$ be the set of all probability measures on $\Omega$
equipped with the standard topology
(the weak$^*$ topology on $\PPP(\Omega)$
identified with a subset of $C(\Omega)'$
by a Riesz representation theorem,
Theorem 7.4.1 in \cite{dudley:2002};
this is also known as the topology of weak convergence
in the case of metrizable $\Omega$).
The subset $\PPPfin(\Omega)$ of $\PPP(\Omega)$
consists of all probability measures in $\PPP(\Omega)$
concentrated on a finite subset of $\Omega$.
If $\Omega$ is finite,
$\PPP(\Omega)=\PPPfin(\Omega)$ can be identified
with an $(\lvert\Omega\rvert-1)$-dimensional simplex
(see below)
in Euclidean space
equipped with the standard Euclidean distance and topology.

Theorem \ref{thm:continuous} will be a statement
about the following perfect-information game
involving three players:

\bigskip

\noindent
\textsc{Continuous game}

\smallskip

\noindent
\textbf{Players:} Sceptic, Forecaster, Reality

\smallskip

\noindent
\textbf{Protocol:}

\parshape=8
\IndentI   \WidthI
\IndentI   \WidthI
\IndentII  \WidthII
\IndentIII \WidthIII
\IndentIII \WidthIII
\IndentII  \WidthII
\IndentII  \WidthII
\IndentII  \WidthII
\noindent
$\K_0 := 1$.\\
FOR $n=1,2,\dots$:\\
  Sceptic announces a function $S_n:\Omega\times\PPP(\Omega)\to\bbbr$\\
    which is lower semicontinuous in the second argument\\
    and satisfies $\int_{\Omega}S_n(\omega,p)p(\dd\omega)\le0$ for all $p\in\PPPfin(\Omega)$.\\
  Forecaster announces $p_n\in\PPP(\Omega)$.\\
  Reality announces $\omega_n\in\Omega$.\\
  $\K_n := \K_{n-1} + S_n(\omega_n,p_n)$.

\smallskip

\noindent
\hangindent=7mm

\smallskip

\noindent
\hangindent=7mm
\textbf{Winner:}
Forecaster wins if Sceptic's capital $\K_n$ stays bounded.

\bigskip

\noindent
(For $p\in\PPPfin(\Omega)$,
the integral $\int_{\Omega}S_n(\omega,p)p(\dd\omega)$
is interpreted as a sum,
and so $S_n(\omega,p)$ is not required to be measurable in $\omega$.)

Intuitively,
on each round of the game
Forecaster is asked to give a probability forecast $p_n$
for the outcome $\omega_n$ to be chosen by Reality.
Sceptic is testing the forecasts $p_n$ by gambling against them.
Forecaster wins the game if Sceptic does not detect serious disagreement
between Forecaster and Reality.

The continuous game is stated here
in the form that makes Theorem \ref{thm:continuous} as strong as possible.
In typical applications in prediction with expert advice and algorithmic information theory,
Sceptic's move $S_n(\omega,p)$
is lower semicontinuous jointly in $(\omega,p)\in\Omega\times\PPP(\Omega)$
and measurable in $\omega$;
the condition $\int_{\Omega}S_n(\omega,p)p(\dd\omega)\le0$
is required to hold for all $p\in\PPP(\Omega)$.
Furthermore,
there is an important restriction imposed on Sceptic:
he must choose $S_n$ so that his capital remains nonnegative
($\K_n \ge 0$) no matter how the other players move
(in particular, the function $S_n$ must be bounded below).
Theorem \ref{thm:continuous}, however,
does not depend on these further assumptions.


The following result was stated (in different terms)
by Levin \cite{\Levin}.
\begin{theorem}\label{thm:continuous}
  Forecaster has a strategy in the continuous game
  that guarantees $\K_0\ge\K_1\ge\K_2\ge\cdots$.
\end{theorem}
In other words,
not only Sceptic does not detect
serious disagreement between Forecaster and Reality,
he does not detect any disagreement at all.

We will reproduce Levin's 
original proof,
as detailed by G\'acs \cite{gacs:2005}, Section 5;
for a different proof
(essentially a reference to Ky Fan's minimax theorem,
\cite{agarwal/etal:2001}, Theorem 11.4)
under stronger assumptions,
see \cite{\DFVIII}, Section 3.

A set $v_1,\ldots,v_M$ of points in a Euclidean (or topological vector) space
is \emph{affinely independent} if,
for all real numbers $\lambda_1,\ldots,\lambda_M$,
\begin{equation*}
  \sum_{m=1}^M \lambda_m v_m = 0
  \text{ and }
  \sum_{m=1}^M \lambda_m = 0
  \text{ imply }
  \lambda_1=\cdots=\lambda_M=0.
\end{equation*}
The convex hull of such $v_1,\ldots,v_M$,
denoted $\co(v_1,\ldots,v_M)$,
is called a \emph{simplex}
or, more fully, an \emph{$(M-1)$-dimensional simplex}.
The proof of Theorem \ref{thm:continuous}
will use the following result due to Knaster, Kuratowski, and Mazurkiewicz
(\cite{knaster/etal:1929}; see also \cite{agarwal/etal:2001}, Theorem 11.2).
\begin{Theorem}{KKM Theorem}
  Let $F_1,\ldots,F_M$ be closed subsets of a simplex $\co(v_1,\ldots,v_M)$.
  Suppose that for all
  $1\le k\le M$ and $1\le m_1\le\cdots\le m_k\le M$
  we have
  \begin{equation*}
    \co
    \left(
      v_{m_1},\ldots,v_{m_k}
    \right)
    \subseteq
    F_{m_1}\cup\cdots\cup F_{m_k}.
  \end{equation*}
  Then $F_1\cap\cdots\cap F_M\ne\emptyset$.
\end{Theorem}
\begin{Proof}{of Theorem \ref{thm:continuous}}
  Fix a round $n$ of the game and set $S:=S_n$.
  For every $\omega\in\Omega$,
  let $F_{\omega}$ be the closed set
  \begin{equation*}
    F_{\omega}
    :=
    \left\{
      p\in\PPP(\Omega)
      \st
      S(\omega,p)\le0
    \right\}.
  \end{equation*}
  It suffices to show that for every finite set of points $\omega_1,\ldots,\omega_M$
  we have
  \begin{equation}\label{eq:nonempty}
    F_{\omega_1} \cap\cdots\cap F_{\omega_M}
    \ne
    \emptyset.
  \end{equation}
  Indeed, the compactness of $\Omega$
  implies the compactness of $\PPP(\Omega)$
  (combine Alaoglu's theorem, Problem 9 in Section 6.1 of \cite{dudley:2002},
  with the weak$^*$ closeness of $\PPP(\Omega)$ in $C(\Omega)'$,
  following from \cite{dudley:2002}, Theorems 7.1.5 and 2.6.3).
  Therefore, if every finite subset of the family
  $\{F_{\omega}\st \omega\in\Omega\}$
  of closed sets
  has a non-empty intersection,
  then the whole family has a nonempty intersection,
  and any of the measures in this intersection can be taken as $p_n$.

  To show (\ref{eq:nonempty}),
  let $\PPP(\omega_1,\ldots,\omega_M)$ be the set of probability measures
  concentrated on $\{\omega_1,\ldots,\omega_M\}$.
  If $p\in\PPP(\omega_1,\ldots,\omega_M)$,
  the inequality $\int S(\omega,p)p(\dd\omega)\le0$
  implies $S(\omega_m,p)\le0$ for some $m\in\{1,\ldots,M\}$.
  Hence $\PPP(\omega_1,\ldots,\omega_M)\subseteq F_{\omega_1}\cup\cdots\cup F_{\omega_M}$,
  and the same holds for every subset of the indices $\{1,\ldots,M\}$.
  The KKM theorem now implies (\ref{eq:nonempty}).
  \qedtext 
\end{Proof}

\section{Randomized defensive forecasting}\label{sec:randomized}

Let $\PPPfin(\PPP(\Omega))$ be the set of all probability measures on $\PPP(\Omega)$
concentrated on a finite subset of $\PPP(\Omega)$.
For each $P\in\PPPfin(\PPP(\Omega))$,
let $\D(P)\subseteq\PPP(\Omega)$ be the smallest finite set in $\PPP(\Omega)$ of $P$-probability one.

Our result about randomized defensive forecasting
concerns the following perfect-information game involving four players:

\bigskip

\noindent
\textsc{Randomized game}

\smallskip

\noindent
\textbf{Players:} Sceptic, Forecaster, Reality, Random Number Generator

\smallskip

\noindent
\textbf{Protocol:}

\parshape=12
\IndentI   \WidthI
\IndentI   \WidthI
\IndentI   \WidthI
\IndentII  \WidthII
\IndentIII \WidthIII
\IndentIII \WidthIII
\IndentII  \WidthII
\IndentII  \WidthII
\IndentII  \WidthII
\IndentII  \WidthII
\IndentII  \WidthII
\IndentII  \WidthII
\noindent
$\K_0 := 1$.\\
$\F_0 := 1$.\\
FOR $n=1,2,\dots$:\\
  Sceptic announces a function $S_n:\Omega\times\PPP(\Omega)\to\bbbr$\\
    which is continuous in the first argument $\omega\in\Omega$\\
    and satisfies $\int_{\Omega}S_n(\omega,p)p(\dd\omega)\le0$ for all $p\in\PPP(\Omega)$.\\
  Forecaster announces $P_n\in\PPPfin(\PPP(\Omega))$.\\
  Reality announces $\omega_n\in\Omega$.\\
  Forecaster announces a function $f_n:\PPP(\Omega)\to\bbbr$
    such that $\int_{\PPP(\Omega)} f_n \dd P_n\le0$.\\
  Random Number Generator announces $p_n\in\D(P_n)$.\\
  $\K_n := \K_{n-1} + S_n(\omega_n,p_n)$.\\
  $\F_n := \F_{n-1} + f_n(p_n)$.

\smallskip

\hangindent=7mm
\noindent
\textbf{Restriction on Sceptic:}
Sceptic must choose $S_n$
(continuous, and so Baire measurable, in its first argument)
so that his capital remains nonnegative
($\K_n \ge 0$) no matter how the other players move
(in particular, the function $S_n$ must be bounded below).

\smallskip

\hangindent=7mm
\noindent
\textbf{Restriction on Forecaster:}
Forecaster must choose his moves
so that his capital remains nonnegative
($\F_n \ge 0$) no matter how the other players move.

\smallskip

\hangindent=7mm
\noindent
\textbf{Winner:}
Forecaster wins if either (i) his capital $\F_n$ tends to infinity
or (ii) Sceptic's capital $\K_n$ stays bounded.

\bigskip

\noindent
(Since $\int_{\PPP(\Omega)} f_n \dd P_n=\int_{\D(P_n)} f_n \dd P_n$ is a sum,
its existence does not depend on the measurability of $f_n$.
However, by the Tietze--Urysohn theorem, Theorem 2.1.8 in \cite{engelking:1989},
$f_n$ can be chosen continuous and, therefore, Baire measurable;
the Tietze--Urysohn theorem is applicable since every compact is normal,
\cite{engelking:1989}, Theorem 3.1.9.)

Forecaster is now allowed to randomize,
and it is Random Number Generator who picks the actual forecast $p_n$
from Forecaster's randomized forecast $P_n$.
As before,
Sceptic is testing the forecasts $p_n$ by gambling against them.
To make sure that Random Number Generator performs his duty
of producing random-looking $p_n$,
Forecaster is allowed to gamble against Random Number Generator's choices.
Forecaster wins the game if he either discredits Random Number Generator
or Sceptic does not detect serious disagreement between the forecasts and the outcomes.

In the case of finite $\Omega$,
the only restriction on Sceptic's move $S_n$
is $\int_{\Omega}S_n(\omega,p)p(\dd\omega)\le0$, $\forall p\in\PPP(\Omega)$.
We will see that in this case
Theorem \ref{thm:randomized} will remain true
even if $f_n$ is required to be a linear function
on the simplex $\PPP(\Omega)$.

The following is the randomized counterpart of Theorem \ref{thm:continuous}.
\begin{theorem}\label{thm:randomized}
  For any $\epsilon>0$
  and any sequence $\AAA_1,\AAA_2,\ldots$ of open covers of the outcome space $\Omega$,
  Forecaster has a strategy in the randomized game
  that guarantees:
  \begin{itemize}
  \item
    $\K_n \le (1+\epsilon)\F_n$ for each $n$;
  \item
    $\D(P_n)$ lies completely in one element of $\AAA_n$;
  \item
    $
      \lvert \D(P_n)\rvert
      \le
      \lvert\Omega\rvert
    $.
  \end{itemize}
\end{theorem}
The last item,
$
  \lvert \D(P_n)\rvert
  \le
  \lvert\Omega\rvert
$,
is of interest only in the case of finite $\Omega$:
it holds trivially when $\Omega$ is infinite.

Before discussing the intuition behind Theorem \ref{thm:randomized}
we restate the second item in a more intuitive form
assuming that $\Omega$ is finite
and $\dist$ is the Euclidean distance on the simplex $\PPP(\Omega)$.
(More generally, $\Omega$ can be assumed a compact metric space
and $\dist$ be, e.g., the Prokhorov metric on $\PPP(\Omega)$;
see, e.g., \cite{billingsley:1968}, Appendix III, Theorem 6.)
\begin{corollary}\label{cor:randomized}
  Suppose $\Omega$ is finite (or a metric compact).
  For any $\epsilon>0$
  and any sequence $\epsilon_1,\epsilon_2,\ldots$ of positive real numbers,
  Forecaster has a strategy in the randomized game
  that guarantees:
  \begin{itemize}
  \item
    $\K_n \le (1+\epsilon)\F_n$ for each $n$;
  \item
    the diameter of $\D(P_n)$ is at most $\epsilon_n$:
    \begin{equation*}
      \diam \D(P_n)
      :=
      \sup_{p,q\in\D(P_n)}
      \dist(p,q)
      =
      \max_{p,q\in\D(P_n)}
      \dist(p,q)
      \le
      \epsilon_n;
    \end{equation*}
  \item
    $
      \lvert\D(P_n)\rvert
      \le
      \lvert\Omega\rvert
    $.
  \end{itemize}
\end{corollary}
The condition $\K_n \le (1+\epsilon)\F_n$
says that Forecaster can guarantee $\F_n \ge \K_n$ to any approximation required,
i.e., every pound gained by Sceptic can be attributed
to the poor performance of Random Number Generator.
The condition $\diam\D(P_n)\le\epsilon_n$
shows that already a tiny amount of randomization is sufficient;
as already mentioned,
a similar observation was made by Kakade and Foster \cite{kakade/foster:2004}.

\begin{Proof}{of Theorem \ref{thm:randomized}}
  We will repeatedly use the fact that $\PPP(\Omega)$ is paracompact
  (\cite{engelking:1989}, Theorem 5.1.1).
  The stronger condition that $\PPP(\Omega)$ is compact
  will only be used in a reference to Theorem \ref{thm:continuous}.

  Fix a round $n$ of the game.
  Let $\delta>0$ be a small constant
  (how small will be determined later).
  For each $p\in\PPP(\Omega)$ set
  \begin{equation}\label{eq:A}
    A_p
    :=
    \left\{
      q\in\PPP(\Omega)
      \st
      \int_{\Omega} S_n(\omega,p) q(\dd\omega)
      <
      \delta
    \right\};
  \end{equation}
  notice that $p\in A_p$
  and that $A_p$ is an open set.
  Let $\BBB$ be any open star refinement of $\AAA_n$
  (it exists by \cite{engelking:1989}, Theorem 5.1.12, (i) and (iii)),
  let $\CCC$ be any locally finite open refinement of $\BBB$
  (it exists by the definition of paracompactness),
  and let $B_p$ be the intersection of $A_p$
  with an arbitrary element of $\CCC$ containing $p$.
  Notice that the $B_p$ form an open cover of $\PPP(\Omega)$.
  If $\Omega$ is finite,
  replace $\{B_p\}_{p\in\PPP(\Omega)}$
  by its open shrinking of order $\lvert\Omega\rvert-1$
  (it exists by the Dowker theorem, Theorem 7.2.4 in \cite{engelking:1989},
  since $\Omega$ is normal, Theorem 3.1.9 in \cite{engelking:1989},
  and $\dim(\PPP(\Omega))=\lvert\Omega\rvert-1$,
  \cite{engelking:1989}, Theorem 7.3.19);
  we will use the same notation $\{B_p\}_{p\in\PPP(\Omega)}$ for the shrinking.
  Let $\{f_s\}_{s\in S}$ be a locally finite partition of unity
  subordinated to the open cover $\{B_p\}_{p\in\PPP(\Omega)}$
  (\cite{engelking:1989}, Theorem 5.1.9).
  For each $s\in S$ choose a $p_s\in\PPP(\Omega)$
  such that $\{p\st f_s(p)>0\}\subseteq B_{p_s}$.
  Set, for $\omega\in\Omega$ and $p\in\PPP(\Omega)$,
  \begin{equation*}
    S^*(\omega,p)
    :=
    \sum_{s\in S}
    S_n(\omega,p_s)
    f_s(p)
  \end{equation*}
  (notice that only a finite number of addends are non-zero,
  so the sum is well-defined).

  In the previous section we were considering
  Sceptic's moves $S_n(\omega,p)$ lower semicontinuous in $p$ and satisfying
  $\int_{\Omega}S_n(\omega,p)p(\dd\omega)\le0$ for all $p\in\PPPfin(\Omega)$.
  It is clear that $S^*(\omega,p)$ is even continuous in $p$;
  let us check that it almost satisfies
  $\int_{\Omega}S^*(\omega,p)p(\dd\omega)\le0$ for all $p\in\PPP(\Omega)$.
  We have:
  \begin{multline}\label{eq:approximate}
    \int_{\Omega}
    S^*(\omega,p)
    p(\dd\omega)
    =
    \int_{\Omega}
    \sum_{s\in S_p}
    S_n(\omega,p_s)
    f_s(p)
    p(\dd\omega)\\
    =
    \sum_{s\in S_p}
    \int_{\Omega}
    S_n(\omega,p_s)
    p(\dd\omega)
    f_s(p)
    \le
    \sum_{s\in S_p}
    \delta
    f_s(p)
    =
    \delta,
  \end{multline}
  where $S_p$ is the finite set of all $s$ for which $f_s(p)>0$;
  the inequality in (\ref{eq:approximate}) uses the fact that $p\in B_{p_s}\subseteq A_{p_s}$
  and the definition (\ref{eq:A}).
  Therefore,
  $\int_{\Omega}S(\omega,p)p(\dd\omega)\le0$ for all $p$,
  where $S:=S^*-\delta$.
  Applying to $S$ the argument given in the proof of Theorem \ref{thm:continuous},
  we can see that there exists $p^*\in\PPP(\Omega)$
  satisfying $S(\omega,p^*)\le0$,
  i.e., $S^*(\omega,p^*)\le\delta$,
  for all $\omega\in\Omega$.

  Make Forecaster select $P_n$ concentrated on the $p_s$ with positive $f_s(p^*)$
  and assigning weight $f_s(p^*)$ to each of these $p_s$.
  This will ensure that $P_n$ is concentrated on a finite subset, $\D(P_n)$,
  of an element of $\AAA_n$
  and that $\lvert\D(P_n)\rvert\le\lvert\Omega\rvert$.

  The rest of the proof proceeds similarly to the proof of Theorem 3 in \cite{\DFI}.
  Let $\delta$ be $\epsilon 2^{-n}$ or less.
  This will ensure
  \begin{equation}\label{eq:minimax}
    \int
      S_n(\omega,p)
    P_n(\dd p)
    \le
    \epsilon 2^{-n}
  \end{equation}
  for all $\omega\in\Omega$.
  Let Forecaster's strategy further tell him
  to use as his second move the function $f_n$ given by
  \begin{equation}\label{eq:f}
    f_n(p)
    :=
    \frac{1}{1+\epsilon}
    \left(
      S_n(\omega_n,p)
      -
      \epsilon 2^{-n}
    \right)
  \end{equation}
  for $p \in \D(P_n)$ and defined arbitrarily for $p \notin \D(P_n)$.
  The condition $\int f_n \dd P_n\le0$ is then guaranteed by~(\ref{eq:minimax}).

  It remains to check $\K_n \le (1+\epsilon) \F_n$
  (this will also establish that $\F_n$ is never negative).
  This can be done by a formal calculation
  (as in the proof of Theorem 3 in \cite{\DFI}),
  but I prefer the following intuitive picture.
  We would like Forecaster to use
  $
    f_n(p)
    :=
    S_n(\omega_n,p)
    -
    \epsilon 2^{-n}
  $
  (for $p\in\D(P_n)$)
  as his second move;
  this would always keep his capital $\F_n$
  above $\K_n-\epsilon$.
  To make sure that $\F_n$ is never negative,
  Forecaster would have to start with initial capital $\F_0=1+\epsilon$,
  which, moreover, would lead to $\F_n\ge\K_n$, $\forall n$;
  our protocol, however, requires $\F_0=1$.
  Therefore,
  Forecaster's strategy has to be scaled down to the initial capital $1$,
  leading to (\ref{eq:f});
  $\F_n\ge\K_n$ becomes $(1+\epsilon)\F_n\ge\K_n$.
  (Scaling down a strategy to a smaller initial capital
  means that the player multiplies the strategy's moves
  by the same factor as he has multiplied the initial capital,
  thus assuring that the capital on succeeding rounds
  is also multiplied by this factor.)
  \qedtext
\end{Proof}

\begin{corollary}
  Forecaster has a winning strategy
  in the randomized game.
\end{corollary}

\begin{proof}
  We are required to show that for every legal strategy $\SSS$ for Forecaster,
  we can construct another legal strategy $\SSS^*$ such that whenever $\SSS$'s capital
  is unbounded, $\SSS^*$'s tends to infinity.
  I will reproduce a simple construction (which I learned from Shen)
  given in \cite{\DFI}, the proof of Theorem 3.
  (For a more efficient, in certain respects, construction
  see \cite{shafer/vovk:2001}, Lemma 3.1;
  an even better construction has been recently devised by Vereshchagin and Shen.)

  We choose some number larger than $1$, say $2$.
  Starting, as the game requires,
  with initial capital $1$ for Forecaster,
  we have him play $\SSS$ until its capital exceeds $2$.
  Then he sets aside $1$ of this capital and continues with a rescaled version of $\SSS$,
  scaled down to the reduced capital. 
  When the capital again exceeds $2$, he again sets aside $1$, and so forth.
  The money set aside,
  which is part of the capital earned by this strategy, grows without bound.
  \qedtext
\end{proof}

Theorem \ref{thm:randomized} imposes a condition of continuity on Sceptic's move $S_n$
whereas Theorem \ref{thm:continuous} only requires lower semicontinuity
(in a different argument).
A natural question is whether we can relax the former condition.
The key point in the proof of Theorem \ref{thm:randomized}
where the continuity of $S_n$ in $\omega$ is used
is the claim that the set (\ref{eq:A}) is open.
This claim will still be true
if $S_n$ is only required to be upper semicontinuous in $\omega$,
at least when $\Omega$ is a metric compact.
We did not pursue this generalization
since it can be deduced from Theorem \ref{thm:randomized} as a corollary
(Corollary \ref{cor:upper-semicontinuous} below).

Let us say that a real-valued function $f$ on $\Omega$ is \emph{strongly upper semicontinuous}
if there is a monotonic sequence of bounded above real-valued functions $f_1\ge f_2\ge\cdots$ on $\Omega$
that converges to $f$ everywhere.
For 
metric compacts,
this requirement coincides with upper semicontinuity
(\cite{engelking:1989}, Problems 1.7.15(c) and 3.12.23(g)),
but in general it is stronger
(\cite{engelking:1989}, Problems 1.7.14(a), 1.7.15(c), and 3.12.23(g)).

\begin{corollary}\label{cor:upper-semicontinuous}
  Theorem \ref{thm:randomized} will continue to hold
  if the condition that $S_n(\omega,p)$
  be continuous in $\omega\in\Omega$
  in the randomized game
  is relaxed to the condition that $S_n(\omega,p)$
  be strongly upper semicontinuous in $\omega\in\Omega$.
\end{corollary}

\begin{proof}
  The proof proceeds similarly to the end of the proof of Theorem \ref{thm:randomized}.
  Let $\epsilon'$ be a small positive constant
  (we will need $(1+\epsilon')^2\le1+\epsilon$).
  Fix, for a moment, a round $n$ of the game.
  By the monotone convergence theorem
  and the definition of strong upper semicontinuity,
  there exists a function $S'_n:\Omega\times\PPP(\Omega)\to\bbbr$
  such that $S'_n\ge S_n$,
  $S'_n(\omega,p)$ is continuous in $\omega\in\Omega$,
  and
  $
    \int_{\Omega}
      S'_n(\omega,p)
    p(\dd\omega)
    \le
    \epsilon' 2^{-n}
  $
  for all $p\in\PPP(\Omega)$.
  Theorem \ref{thm:randomized} is applicable to Sceptic's move
  \begin{equation*}
    S''_n
    :=
    \frac{1}{1+\epsilon'}
    \left(
      S'_n - \epsilon' 2^{-n}
    \right)
  \end{equation*}
  on round $n$,
  for each $n=1,2,\ldots$,
  and it asserts the existence of a strategy for Forecaster
  ensuring
  \begin{equation*}
    \K_n
    \le
    \K'_n
    \le
    (1+\epsilon')\K''_n
    \le
    (1+\epsilon')^2 \F_n,
  \end{equation*}
  where $\K'$ is the capital corresponding to the strategy $S'$
  (formally, $\K'_n:=1+\sum_{i=1}^{n}S'_i(\omega_i,p_i)$)
  and $\K''$ is the capital corresponding to the strategy $S''$.
\qedtext
\end{proof}

Theorem \ref{thm:randomized} is a general form of Theorem 5 in \cite{\DFI}
(that theorem is not part of the journal version).
This note is self-contained from the mathematical point of view,
but for further motivation behind Theorem \ref{thm:randomized}
the reader is referred to \cite{\DFI}.

\section{Discussion}\label{sec:discussion}

This note assumes that the outcome space $\Omega$ is a compact.
This assumption is not as restrictive as it seems
since a wide range of topological spaces have compactifications
that are still ``nice'' topological spaces
(cf.\ \cite{\AAVI}, the subsection on pp.~4--5).
It appears that implications of this fact for prediction with expert advice
(see, e.g., \cite{\DFVIII})
deserve to be explored.

\subsection*{Acknowledgments}

This work was motivated by Sasha Shen's question.
It was partially supported by EPSRC through grant EP/F002998/1.

\ifWP
  \DFGTlastpage
\fi
\end{document}